%% file: ijcai24.tex
\documentclass{article}
\usepackage{colortbl}
\usepackage{ijcai24}

\usepackage{times}
\usepackage{soul}
\usepackage{url}
\usepackage[hidelinks]{hyperref}
\usepackage[utf8]{inputenc}
\usepackage[small]{caption}
\usepackage{graphicx}
\usepackage{amsmath}
\usepackage{amsthm}
\usepackage{booktabs}
\usepackage{algorithm}
\usepackage{algorithmic}
\usepackage[switch]{lineno}

\usepackage{xcolor}

 \usepackage{multirow}
\usepackage{subfig}
 \usepackage{tabularx}
\usepackage{subcaption} 
\usepackage {hyperref}
\usepackage{booktabs} 
\usepackage {amsthm} 
\usepackage{txfonts}
\usepackage{subcaption}
\usepackage{natbib}
\setcitestyle{comma,  numbers, sort&compress,  super}
\floatname{algorithm}{Algorithm}

\newtheorem{theorem}{Theorem}
\newtheorem{lemma}{Lemma}
\newtheorem{assumption}{Assumption}

\title{Approximated Likelihood Ratio: A Forward-Only and Parallel Framework for Boosting Neural Network Training}

\author{Zeliang Zhang$^{1}$\footnotemark[1]\, Jinyang Jiang$^{2}$\footnotemark[1]\, Zhuo Liu$^{1}$\, Susan Liang$^{1}$\,  Yijie Peng$^{2}$\, Chenliang Xu$^{1}$\ \\
$^{1}$ University of Rochester \quad
$^{2}$Peking University\\}

\begin{document}

\maketitle
\renewcommand{\thefootnote}{\fnsymbol{footnote}} 
\footnotetext[1]{\ These authors contributed equally to this work. Listing order is random. }

\begin{abstract}
Efficient and biologically plausible alternatives to backpropagation in neural network training remain a challenge due to issues such as high computational complexity and additional assumptions about neural networks, which limit scalability to deeper networks. The likelihood ratio method offers a promising gradient estimation strategy but is constrained by significant memory consumption, especially when deploying multiple copies of data to reduce estimation variance. In this paper, we introduce an approximation technique for the likelihood ratio (LR) method to alleviate computational and memory demands in gradient estimation. By exploiting the natural parallelism during the backward pass using LR, we further provide a high-performance training strategy, which pipelines both the forward and backward pass, to make it more suitable for the computation on specialized hardware. Extensive experiments demonstrate the effectiveness of the approximation technique in neural network training. This work underscores the potential of the likelihood ratio method in achieving high-performance neural network training, suggesting avenues for further exploration.

\end{abstract}

\section{Introduction}
Deep neural networks (DNNs) have achieved remarkable success in various applications,  including object detection~\citep{li2022exploring},  text generation~\citep{li2022diffusion},  and more~\citep{gehlot2022application}. However,  the training of neural networks remains an expensive and time-consuming endeavor with the increasing size of models~\citep{athlur2022varuna},  exemplified by influential architectures like CLIP~\citep{li2022clip},  Diffusion~\citep{croitoru2023diffusion},  and GPT~\citep{lund2023chatting}. For instance,  GPT-3,  with its vast number of 175 billion parameters,  demands an extensive 34 days of training from scratch using 1,024 GPUs~\citep{ding2023mirage}.  In the past years,  there have been many works to train neural networks with higher efficiency,  e.g.,  the wide use of GPUs~\citep{wang2022lightseq2},   mixed precision acceleration using GPU tensor cores~\citep{liu2022high},  distributed training~\citep{agarwal2022utility},  sparking a huge interest in both academia and industry.

\input{Figs/intro}

Despite continuous efforts to provide efficient DNN training,  the majority of current research {focuses} on largely the same optimization algorithm, i.e., gradient-descent combined with backpropagation (BP).
While {BP} provides an opportunity to globally optimize all model parameters, it also comes with prohibitive problems known as the \textit{weight transport}~\citep{grossberg1987competitive} and \textit{update locking}~\citep{jaderberg2017decoupled}, creating a bottleneck when training is assigned to specialized hardware for further acceleration. Moreover, the biophysical viability of {BP} comes into question as this intricate process seems ill-suited for implementation by actual neurons within the human brain~\citep{zarkeshian2022photons}, hindering the training efficiency by nature. Hence, exploring alternative approaches and gaining insights into the learning mechanisms in biological neural networks could inspire novel paradigms for DNN training.

In the past few years, the pursuit of biologically plausible efficient training techniques has attracted substantial attention. The HSIC {bottleneck}~\citep{ma2020hsic} maximizes the {Hilbert-Schmidt} independence criterion~\citep{wang2021learning} to improve the independence between layers,  thus achieving the layer-wise training but having a performance degradation due to the reduction of layer dependence.  The feedback alignment (FA)~\citep{nokland2016direct} and neural tangent kernel (NTK)~\citep{jacot2018neural} employ additional intermediate variables to break the chain rule in BP but suffer from instability and computation efficiency. Another avenue involves the application of {perturbation-based} methods~\citep{jiang2023training}, exemplified by the evolution strategy (ES)~\citep{salimans2017evolution} and likelihood ratio (LR)~\citep{peng2022new} methods. ES selects the best-perturbed weights for evolution while LR estimates the gradients for {iterative descent} by perturbing neuron outputs. A recent study identified that ES is one of the special cases of LR~\citep{jiang2023training}, where ES implicitly estimates the gradients by perturbing weights, and proposed a hybrid strategy of LR and ES to improve the performance.  In our work, we group conventional LR, ES, and hybrid as the LR family.


Among these methods,  the LR family stands out as a compelling choice for investigating efficient neural network training in the absence of BP. There are several reasons for this preference. {\underline{First},  LR has no assumption for layer independence,  which coordinates the training across entire neural networks, allowing for the learning of features from top to bottom for better performance and interoperability.} \underline{Second}, it is designed to estimate the gradient of the original loss function with respect to model parameters, obviating the need for any additional modules or criteria. This inherent simplicity makes it easily compatible with existing deep learning techniques originally developed for BP. \underline{Third},  LR has exhibited natural improvements in terms of robustness, thus holding the potential to create steadfast DNNs resilient to real-world adversarial challenges~\citep{peng2022new}.

Nonetheless, the utilization of LR encounters a constraint arising from significant memory consumption resulting from the requirement of numerous copies, essential to reduce estimation variance. For example, in Fig. \ref{fig:intro}, when training a ResNet-9 model on the CIFAR-100 dataset,  with increasing the {number of data copies} from $100$ to $500$, LR exhibits a notable improvement in average gradient estimation accuracy from $0.17$ to $0.28$, measured by cosine similarity with the true gradient, thus contributing to higher classification accuracy from $18.6\%$ to $38.5\%$. Unfortunately, the use of a larger number of copies for more accurate gradient estimation is curtailed by memory limitations. In the approaching post-Moore era, with the development of the hardware capacity slowing down, this memory-imposed constraint manifests itself as a performance gap between the biologically plausible LR method and the conventionally used BP approach.



To address these challenges and unleash the potential of  LR for high-performance neural network training,    we propose an approximated {LR} method {to compute the ascent direction}. Specifically, with a {convergence} guarantee, we just {take} the sign of intermediate variables used by LR for gradient estimation,  which largely reduces memory consumption during the training process,  allowing a larger number of copies for higher accuracy of gradient estimation and better task performance. Besides,  we implement neural network training with distributed computation techniques {on} both the data and layer level and propose a pipeline strategy, which theoretically and practically reduces the training time compared with BP. Our contributions are as follows: 
\begin{itemize}
    \item We propose a LR training framework, in which the approximation technique is invoked to estimate the {ascent direction}  for model optimization{,} and give a convergence analysis of the proposed optimization technique. 
    \item We analyze the parallelism in the proposed approximated LR training framework and present a high-performance implementation and a pipeline strategy. 
    \item We evaluate the correctness and effectiveness on a wide range of neural network architectures and datasets. Extensive experiments demonstrate the effectiveness of the approximated LR and highlight the great potential for the alternative of BP.
\end{itemize}

\section{{LR} Method for {DNN} Training}
We denote $L$ as the number of layers in {DNN}s and $m_{l}$ as the number of neurons in the $l$-th neural layer,  $l \in [1,  2,  ...,  L]$. For the input ${X}^{(0)} \in \mathbb{R}^{m_{0}}$,  we have the output of $l$-th layer ${X}^{(l)}=[{x}^{(l)}_{1},  {x}^{(l)}_{2},  ...,  {x}^{(l)}_{m_{l}}] \in \mathbb{R}^{m_{l}}$.

Suppose we have $N$ inputs for the network,  denoted as ${X}^{(0)}(n)$,  $n=1,  2,  ...,  N$. For the $n$-th input,  the $i$-th output at the $l$-th layer can be given by
\begin{align*}
    {x}^{(l)}_{i}(n)=\varphi({v}^{(l)}_{i}(n)),\quad {v}^{(l)}_{i}(n)=\sum_{j=0}^{m_{l-1}}\theta^{(l)}_{i, j}{x}^{(l-1)}_{j}(n)+{\sigma}^{(l)}_{i}{\varepsilon}^{(l)}_{i}(n), 
\end{align*}
where $\theta^{(l)}_{i, j}$ is the synaptic weight in the $i$-th neuron for the {$j$-th} input at the $l$-th layer,  {${v}^{(l)}_{i}(n)$} is the $i$-th logit output at the $l$-th layer {for the $n$-th data},  $\varphi$ is the activation function,  ${\varepsilon}^{(l)}_{i}(n)$ is an independent random noise following standard normal distribution added to the $i$-th neuron at the $l$-th layer,  and $\sigma_i^{(l)}$ is the standard deviation to scale up or down the noise. We let ${x}_0^{(l)}(n)\equiv 1${,} and then $\theta^{(l)}_{i, 0}$ {represents} the bias term in the linear operation of the $i$-th neuron at the $l$-th layer.  

{For the $n$-th input ${X}^{(0)}(n)$, we denote the loss as $\mathcal{L}_n(\theta, \sigma)$. In classification tasks, the loss function is usually the cross entropy computed by
\begin{align*}
    \mathcal{L}_n(\theta, \sigma) = -\sum_{i=1}^{m_{L}}{o}_{i}(n)\log\left( p_{i}({X}^{(L)}(n))\right), 
\end{align*}  
where ${O}{(n)}=[{o}_{1}(n),  {o}_{2}(n),  ...,  {o}_{m_{L}}(n)] \in \mathbb{R}^{m_{L}}$ is the corresponding label and $p_{i}({X}^{(L)}(n))=\frac{\exp{({x}^{(L)}_{i}}(n))}{\sum_{j=1}^{m_{L}}\exp{({x}^{(L)}_{j}(n)})}$.}
Training {DNN}s is to solve the following optimization problem:
\begin{align*}
    \min\limits_{(\theta, \sigma)\in\Theta\times\Sigma}\mathcal{L}(\theta, \sigma):=  \frac{1}{N}\sum_{n=1}^N\mathbb{E}\left[\mathcal{L}_n(\theta, \sigma)\right], 
\end{align*}
and a basic approach to solving it is by stochastic gradient descent (SGD) algorithms. Let $\omega=(\theta, \sigma)$ and $\Omega = \Theta\times\Sigma$. The SGD algorithm updates $\omega$ by
\begin{align*}
    \omega_{k+1} = \Pi_{\Omega}(\omega_k-\lambda_kg_k), \quad g_k = \frac{1}{b}\sum_{n\in B_k} g_{k, n}, 
\end{align*}
where $g_{k, n}$ is an unbiased estimator for the gradient of $\mathbb{E}\left[\mathcal{L}_n(\theta, \sigma)\right]$,  $B_k=\{n_k^1, \cdots, n_k^b\}$ is a set of indices in a mini-batch randomly drawn from the $N$ data points,  $\lambda_k$ is the learning rate,  and $\Pi_{\Omega}$ is the projection onto $\Omega$ that bounds the values of $\omega$ in order to achieve convergence of SGD.

{Using the} LR method,  we have an unbiased gradient estimation of $\mathbb{E}\left[\mathcal{L}_n(\theta, \sigma)\right]$ with respect to {DNN} parameters as follows:
\begin{equation}\label{gradient}
\begin{aligned}
    \frac{\partial \mathbb{E}\left[\mathcal{L}_n(\theta, \sigma)\right]}{\partial \theta_{i, j}^{(l)}}&=\mathbb{E}\left[\mathcal{L}_n(\theta, \sigma){x}^{({l-1})}_{j}(n)\frac{\varepsilon^{(l)}_{i}(n)}{{\sigma}^{(l)}_{i}}\right], \\
    \frac{\partial \mathbb{E}\left[\mathcal{L}_n(\theta, \sigma)\right]}{\partial {\sigma}^{(l)}_{i}}&=\mathbb{E}\left[\mathcal{L}_n(\theta, \sigma)\frac{1}{{\sigma}^{(l)}_{i}}\big(\varepsilon^{(l)}_{i}(n)^2-1\big)\right].
\end{aligned}
\end{equation}

\section{Approximated {LR} Method}
\label{sec:simplify_lr}
Drawing inspiration from the innate processes of nature, specifically the utilization of electronic signals ($\pm 1$) to facilitate the learning mechanisms in biological intelligence, our motivation stems from the aspiration to leverage bioelectric signals  to encode feedback information for gradient estimation. This approach seeks to overcome the limitations of LR and enhance its applicability in practical settings.

{In this section}, we first elaborate on our approximation {of} the LR method using sign encoding. Subsequently, we give a theoretical analysis of the convergence. Afterward, we propose a parallel strategy to achieve efficient training. 

\subsection{Approximation Through Sign Encoding}
We substitute the term inside the expectation in the first line of Eq. (\ref{gradient}) which contributes most of the computational burden. The original gradient estimation can be presented as 
\begin{align}\label{org_grad}
    g_k = \frac{1}{b}\sum_{n\in B_k}\mathcal{L}_n(\theta_k, \sigma_k)Z_n, 
\end{align}
where $Z_n = (Z_n^{\theta}, Z_n^{\sigma})$,   $(Z_n^{\theta})_{i, j}^{(l)}={x}^{({l-1})}_{j}(n)\frac{\varepsilon^{(l)}_{i}(n)}{{\sigma}^{(l)}_{i}}$ and $(Z_n^{\sigma})_{i}^{(l)}=\frac{1}{{\sigma}^{(l)}_{i}}(\varepsilon^{(l)}_{i}(n)^2-1)$. Then we propose the following surrogate ascent direction:
\begin{align}\label{sign_grad}
    \tilde{g}_k = \frac{1}{b}\sum_{n\in B_k}\mathcal{L}_n(\theta_k, \sigma_k)(\text{sign}(Z_n^{\theta}), Z_n^{\sigma}), 
\end{align}
where we use the $\text{sign}(\cdot)$ function to encode the  bioelectric signal. The surrogate gradient is integrated into SGD for optimization, \textit{i.e.}, 
\begin{align}\label{new_sgd}
    \omega_{k+1} = \Pi_{\Omega}(\omega_k-\lambda_k\tilde{g}_k).
\end{align}

\subsection{Convergence Analysis}
Denote the flattened $Z_n$ and $\omega_k$ as $[z_n^1, \cdots, z_n^D]^{\top}$ and $[\omega_k^1, \cdots, \omega_k^D]^{\top}$,  where the first $D_0$ dimensional terms correspond to the coordinates of the $\theta_k$ and $1<D_0<D$,  and the rest terms correspond to the coordinates of $\sigma_k$. Similar to LR,  $\tilde{g}_k$ is an unbiased estimator of
\begin{equation}\label{approx_gradient}
    \begin{aligned}
    &\mathcal{J}(\omega_k):=\frac{1}{N}\sum_{n=1}^N\mathbb{E}\left[\mathcal{L}_n(\omega_k)(\text{sign}(Z_n^{\theta}), Z_n^{\sigma})\right]\\
    & = \frac{1}{N}\sum_{n=1}^N\Bigg{[}
    \frac{\partial}{\partial\omega_{k, 1}}\mathbb{E}\left[\frac{\mathcal{L}_n(\omega_k)}{\vert Z_{n, 1}\vert}\right], \cdots, \\
    &\frac{\partial}{\partial\omega_{k, D_0}}\mathbb{E}\left[\frac{\mathcal{L}_n(\omega_k)}{\vert Z_{n, D_0}\vert}\right], 
    \frac{\partial\mathbb{E}\left[\mathcal{L}_n(\omega_k)\right]}{\partial\omega_{k, D_0+1}}, \cdots, 
    \frac{\partial\mathbb{E}\left[\mathcal{L}_n(\omega_k)\right]}{\partial\omega_{k, D}}\Bigg{]}^{\top}.
\end{aligned}
\end{equation}

\input{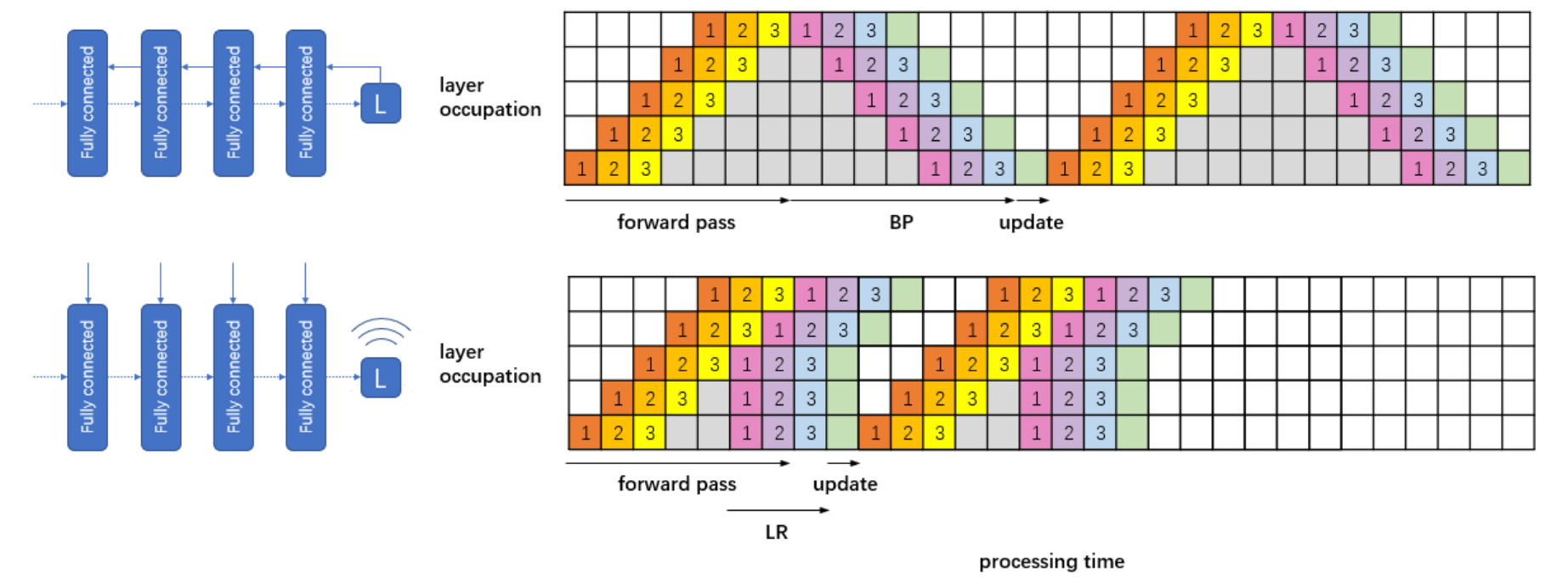}


Now we assume the objective function has a unique equilibrium point $\tilde{\omega}^*\in\Omega$ and discuss the convergence of recursion (\ref{new_sgd}). We define $\mathcal{F}_k=\{\omega_0, \cdots, \omega_k\}$ as the $\sigma$-algebra generated by our algorithm for $k=0, 1, \cdots$. Here we introduce some assumptions before the analysis.
\begin{assumption}\label{a1}
The parameter set $\Omega\subset\mathbb{R}^d$ is closed,  convex, and compact.
\end{assumption}
\begin{assumption}\label{a2}
The cost function $\mathcal{L}^d(\omega_k)$ is continuously differentiable in $\omega_k$,  and convex in $\omega_{k, d}$ for all given $\omega_{k, -d}=[\omega_{k, 1}, \cdots, \omega_{k, d-1}, \omega_{k, d+1}, \cdots\omega_{k, D}]\in\Omega_{-d}\subset\mathbb{R}^{d-1}$.
\end{assumption}
\begin{assumption}\label{a3}
The step-size sequence $\{\gamma_k\}$ satisfies $\gamma_k>0$,  $\sum_{k=0}^{\infty}\gamma_k=\infty$,  $\sum_{k=0}^{\infty}\gamma_k^2<\infty$.
\end{assumption}
\begin{assumption}\label{a4}
The loss value is uniformly bounded,  i.e.,  for all $\omega\in\Omega$,  $\vert\mathcal{L}_n(\omega)\vert\leq M<\infty$ w.p.1.
\end{assumption}

We expect recursion (\ref{new_sgd}) to track an ODE:
\begin{align}\label{proj_ode}
    \dot{\omega}(t) = \tilde{\Pi}_{\Omega}(\mathcal{J}(\omega(t))), 
\end{align}
with $\tilde{\Pi}_{\Omega}(\cdot)$ being a projection function satisfying $\tilde{\Pi}_{\Omega}(\mathcal{J}(\omega(t)))=\mathcal{J}(\omega(t))+p(t)$,  where $p(t)\in-C(\omega(t))$ is the vector with the smallest norm needed to keep $\omega(t)$ in $\Omega$,  and $C(\omega)$ is the normal cone to $\Omega$ at $\omega$. We first establish the unique global asymptotically stable equilibrium for ODE (\ref{proj_ode}).

\begin{lemma}\label{lemma}
    If Assumptions \ref{a1} and \ref{a2} hold,  then $\tilde{\omega}^*$ is the unique global asymptotically stable equilibrium of ODE (\ref{proj_ode}).
\end{lemma}
\begin{proof}
With Assumptions \ref{a1} and \ref{a2},  if $\tilde{\omega}^*\in\Omega^{\circ}$,  then $\mathcal{J}(\tilde{\omega}^*)=0$ and $C(\tilde{\omega}^*)=\{0\}$; if $\tilde{\omega}^*\in\partial\Omega$,  then $\mathcal{J}(\tilde{\omega}^*)$ must lie in $C(\tilde{\omega}^*)$,  so $p(t)=-\mathcal{J}(\tilde{\omega}^*)$.
By the convexity of $\mathcal{L}^d(\omega_k)$ and the assumption that $\tilde{\omega}^*\in\Omega$ is the unique equilibrium point,  $\tilde{\omega}^*\in\Omega$ is the unique equilibrium point of ODE (\ref{proj_ode}).
Take $V(x)=\Vert x-\tilde{\omega}^*\Vert^2$ as the Lyapunov function,  and the derivative is $\dot{V}(x)=2(x-\tilde{\omega}^*)^{\top}(\mathcal{J}(x)+p(t))$.
By Assumption \ref{a2} and \cite{facchinei2003finite},  $(x-\tilde{\omega}^*)^{\top}\mathcal{J}(x)\leq0$ for any $x\neq\tilde{\omega}^*$.
Since $p(t)\in C(x)$,  we have $(x-\tilde{\omega}^*)^{\top}p(t)\leq0$. Therefore,  $\tilde{\omega}^*$ is global asymptotically stable by the Lyapunov Stability Theory \cite{liapounoff2016probleme}.
\end{proof}

To prove that recursion (\ref{new_sgd}) tracks ODE (\ref{proj_ode}),  we apply a convergence theorem in \citep{borkar1997stochastic} as below:
\begin{theorem}\label{theorem1}
Consider the recursion
\begin{align*}
    \omega_{k+1}=\Pi_{\Omega}(\omega_k+\gamma_k(\mathcal{J}(\omega_k)+\delta_k)), 
\end{align*}
where $\Pi_{\Omega}$ is a projection function,  $\mathcal{J}$ is Lipschitz continuous,  $\{\gamma_k\}$ satisfies Assumption \ref{a1},  $\{\delta_k\}$ is random variable sequence satisfying $\sum_k\gamma_k\delta_k<\infty$,  a.s. If ODE (\ref{proj_ode}) has a unique global asymptotically stable equilibrium $\tilde{\omega}^*$,  then the recursion converges to $\tilde{\omega}^*$.
\end{theorem}

Next, we show that the conditions in Theorem \ref{theorem1} can be verified in Theorem \ref{theorem2}.
\begin{theorem}\label{theorem2}
If Assumptions \ref{a1},  \ref{a2},  \ref{a3} and \ref{a4} hold,  then the sequence $\{\omega_k\}$ generated by recursion (\ref{new_sgd}) converges to the unique optimal solution w.p.1.
\end{theorem}
\begin{proof} Recursion (\ref{new_sgd}) can be rewritten as
\begin{align*}
    \omega_{k+1}=\omega_k-\gamma_k\mathcal{J}(\omega_k)+\gamma_k \delta_k, 
\end{align*}
where $\delta_k=\mathcal{J}(\omega_k)-\tilde{g}_k$. Let $M_k=\sum_{i=0}^k \gamma_i\delta_i$. Since $\tilde{g}_k$ is the unbiased estimator of $\mathcal{J}(\omega_k)$,  $\{M_k\}$ is a martingale sequence. We can verify that it is $L^2$-bounded. With Assumptions \ref{a3} and \ref{a4},  we have
\begin{align*}
    \sum_{i=0}^k\gamma_i^2\delta_i^2\leq M^2\sum_{i=0}^k\gamma_i^2<\infty.
\end{align*}
By noticing $\mathbb{E}[\delta_i|\mathcal{F}_i]=0$,  we have 
\begin{align*}
    \mathbb{E}[\gamma_i\delta_i\gamma_j\delta_j]=\mathbb{E}[\gamma_i\delta_i\mathbb{E}[\gamma_j\delta_j|\mathcal{F}_j]]=0, 
\end{align*}
for all $i<j$. Thus $\sup_{k\geq0}\mathbb{E}[M_k^2]<\infty$. From the martingale convergence theorem \citep{durrett2019probability},  we have $M_k\rightarrow M_{\infty}$ w.p.1,  which implies that $\{M_k\}$ is bounded w.p.1. Now all conditions in Theorem 1 are satisfied. Therefore,  it is almost sure that recursion (\ref{new_sgd}) converges to the unique global asymptotically stable equilibrium of ODE (\ref{proj_ode}),  which is the equilibrium point $\tilde{\omega}^*$ by the conclusion of Lemma \ref{lemma}.
\end{proof}


\subsection{Parallel Analysis and Implementations}

 
There are two levels of parallelism in the approximated LR computation,  which we can fully exploit to further accelerate the neural network training,  namely the data-level and layer-level parallelism. 


\paragraph{Data-level parallel.} For a given data $X$,  we use multiple copies to compute the estimated gradient $g$ with reducing the large estimation variance. The computation for different copies is independent and parallel,  which enables us to use distributed computation techniques for acceleration.

\paragraph{Layer-level parallel.} The gradient of the parameter for the update in each layer only depends on the loss value $l$,  the input feature $X$,  and the injected random noise {$\varepsilon$}. The computation of $g$ for different layers is independent and parallel,  which enables us to optimize the layers all in parallel.

The approximated LR computation can be integrated with the data-level and layer-level parallelism,  allowing us to use more copies of data for accurate gradient estimation in training,  thus achieving memory-friendly,  computation-efficient,  and stable neural network training. 


\subsection{Hardware-efficient LR Training Pipeline}

To efficiently manage large batch size training in neural networks, a conventional approach involves partitioning the data into buckets and orchestrating concurrent forward and backward passes~\cite{huang2019gpipe}. This process is exemplified in Fig. \ref{fig:pipeline}, where a $4$-layer model is trained with three data buckets. In this configuration, each computation unit accommodates all layers and the loss computation module. Notably, these computation units adhere to an atomic principle, allowing only one computation task per time step. Nevertheless, due to the inherent reliance on the chain rule of differentiation, a substantial $53.3\%$ redundant space persists throughout the training process, resulting in suboptimal utilization of computation units.

Leveraging the gradient independence between layers in LR, we can enhance the efficiency by parallelizing forward computations in the latter units and gradient estimations in the leading free units, depicted in the lower section of Fig. \ref{fig:pipeline}. This strategy yields both a theoretical speedup of $1.5 \times$, in terms of the expected time steps, compared to BP-based optimization and a remarkable reduction in redundant space to $24\%$.

\section{Evaluations}

To demonstrate the effectiveness of the proposed approximated LR approach, we follow \citep{jiang2023training} to conduct a comprehensive investigation with a diverse range of architectures and datasets. 
We first verify the correctness of the approximation technique, \textit{i.e.}, the sign encoding. Then, we evaluate the feasibility, scalability, and generalizability on image recognition, text classification, and graph node classification tasks with residual neural network (ResNet), long short-term memory (LSTM) network, and spiking neural network (SNN). We further evaluate the running efficiency with our proposed parallel framework. We provide the full experiment setting in Section A of our appendix.




\subsection{Verification Study on the Approximation}
To verify the approximation approach, \textit{i.e.}, sign encoding, we select three zero-order optimization methods as baselines: (i) {LR}~\citep{peng2022new}:  the gradient is estimated by only injecting the noise on the intermediate layer outputs; (ii) {ES}~\citep{salimans2017evolution}: the gradient is estimated by only injecting the noise on the neural network parameters; (iii) {Hybrid}~\citep{jiang2023training}: the gradient is estimated by the hybrid manner of LR and ES, where the ES is used for the first two layers and LR is used for the remaining layers. 

In line with this, we incorporate our proposed approximation technique into each of the studied methods, namely {ALR}, {AES}, and {A-Hybrid}, respectively. The models are trained on the CIFAR-10 dataset for 100 epochs to ensure adequate convergence. The training process employs a batch size of $50$, $200$ copies, and a learning rate of $1 \times 10^{{-3}}$. We split the dataset for training,  validation,  and testing with the ratio $7:2:1$. We report the loss and accuracy for the training and validation datasets. 

\input{Figs/cifar_verification_train}
The training logs, as illustrated in Fig. \ref{fig:veri_cifar}, provide valuable insights into the impact of the approximation technique. Across all models – ALR, AES, and A-Hybrid – the integration of the approximation method showcases effective knowledge acquisition from the data, yielding commendable performances. Notably, ALR achieves a remarkable validation accuracy of $59.5\%$, AES reaches a peak validation accuracy of $43.1\%$, while A-Hybrid attains the highest accuracy at $61.3\%$. Concurrently, the employment of the approximated computation significantly expedites optimization convergence. Analysis of the validation accuracy reveals A-Hybrid's rapid improvement in the initial 25 epochs compared to the conventional Hybrid approach. Moreover, ALR and AES outshine their non-approximated counterparts in terms of performance. However, it's important to acknowledge that the approximated method occasionally demonstrates signs of overfitting and instability, evident in the fluctuating loss values. These observations collectively underscore the validity of the approximated LR method, emphasizing its attributes of accelerated convergence coupled with inherent instability.

\subsection{Performance Evaluation on Larger Datasets}
\paragraph{Evaluation on CIFAR-100 dataset.}
\input{tabs/comp_cifar100} For a more comprehensive study, we train the ResNet-9 on the CIFAR-100 dataset using BP~\citep{rumelhart1986learning} and various non-BP optimization methods, including HSIC~\citep{ma2020hsic}, FA~\citep{nokland2016direct},  ES~\citep{salimans2017evolution}, LR~\citep{peng2022new}, and Hybrid~\citep{jiang2023training}. Among them, HSIC and FA are local learning methods without relying on the gradients, while ES, LR, and Hybrid are gradient estimation optimization methods.  We integrate the approximated technique into different methods, which takes the sign of the update for parameter optimization. The numerical results are shown in Tab. \ref{tab:comp_cifar100}. The gradient estimation methods, especially the LR and Hybrid with the integration of the approximation can achieve comparable performance with BP. Specifically, the Hybrid optimization with the approximated technique only has a $0.3\%$ classification accuracy drop. It surpasses the runner-up method, which belongs to another optimization category, HSIC, with a significant improvement of $32.2\%$. Besides, while the application of approximation technique has a negative impact on the performance of the compared baselines, including the BP, HSIC, and FA, it benefits the gradient estimation methods, especially the LR and ES, with an average improvement of $16.6\%$.

\paragraph{Evaluation on Tiny-ImageNet dataset.}Benefiting from the reduced memory consumption, LR can scale to real-world large datasets. To give a further scalability evaluation, we train the ResNet-12 to classify the Tiny-ImageNet dataset.  We take the supported maximum number of copies for gradient estimation in LR, ES, and Hybrid, as well as the corresponding approximated version. We present the results in Tab. \ref{tab:comp_tiny}.
\input{tabs/comp_tinyImageNet}

By integrating the approximated technique, the gradient estimation methods enjoy a large number of copies, around $4\times$ compared with vanilla LR/ES/Hybrid, contributing to the improvement of the performance.  Among all the baselines, with the loss of global information, the local learning optimizations, including HSIC and FA, lose scalability in deeper neural networks and larger datasets, while the gradient estimation methods use rough gradients to guide the optimization efficiently. It should be noted that the Hybrid method with approximated technique achieves the most comparable performance with BP, which only has a performance drop of $1.5\%$. It thoroughly demonstrates the feasibility and great potential of the approximated gradient estimation methods in training neural networks. 

\paragraph{Generalization to more architectures.}
\input{tabs/evaluation}We provide further evaluation experiments on more architectures, including training LSTM on Ag-News, GAT on Cora, and SNN on Fashion-MNIST. While the local learning methods loss scalability in diverse architecture, we only consider the gradient optimization methods,  ES and LR, as the baseline methods and integrate the approximated technique into them, respectively.  The results are shown in Tab. \ref{tab:multi_architecures}. Compared with baseline methods,  their approximated variants achieve a comparable result on three tasks.  Specifically,  the approximation technique improves the computation and memory efficiency, resulting  $0.8\%$ improvement in the classification accuracy.  The comparable classification accuracy indicates the effectiveness and scalability of the approximated gradient estimation in training various architectures.

\subsection{Evaluation of Running Efficiency}
We conduct experiments to highlight the improved running efficiency. Specifically, the backpropagation method is taken as the baseline for ResNet-5 training, namely the BP. We first apply the approximation technique to LR as the ALR,  where the INT8 data format is used for the storage and computation of the intermediate results benefiting from the sign encoding operation. Next, We integrate the pipeline strategy into both LR and ALR methods, namely LR-P and ALR-P. With an increasing number of copies for gradient estimation, we count corresponding iterations per second for each implementation as the criteria to evaluate the running efficiency. 
\input{Figs/accelerate}

As shown in Fig. \ref{fig:acc},  both the approximation technique and the pipeline strategies benefit the acceleration of the LR computation. Specifically, with an increasing number of copies for gradient estimation, ALR presents an improving superiority versus the LR with a maximum speedup of  $1.22\times$. At the same time, with the power of pipeline technique, ALR-P achieves an average speedup of $2.04 \times$ versus the baseline LR method. Besides, we reduce the additional memory usage for ALR  with $4 \times$ compared with LR due to the use of INT8 format, which brings an improvement of gradient accuracy improvement with the same computation power. It should also be noted that within a certain number of copies, the parallelism of LR brings an efficiency improvement to neural network training compared with BP. By further improving the number of copies for more accurate estimation, the efficiency advantage will gradually diminish to a disadvantage due to the consumption of data copy and limited memory. We call for a more in-depth theory study to alleviate this problem. 




\subsection{Gradient Estimation Performance Analysis}  
\input{Figs/trend}

In this section, we provide more intuitive explanations from the perspective of the gradient estimation accuracy and optimization trajectory. 

\paragraph{Gradient Cosine Similarity.}For ResNet-5, we use the cosine similarity to perform a comparison of the gradient estimation accuracy between Hybrid and A-Hybrid as shown in Fig. \ref{fig:gradient_acc}. We conduct $10$ independent experiments, using the mean value to plot the solid lines and employing the shadow region to indicate the variance between the macro experiments, where the region width represents the variance or uncertainty. 

From left to right, we respectively present the cosine similarities of gradient estimations by Hybrid and BP,  A-Hybrid and BP. We can see that 1) the approximated technique reduces the gradient estimation accuracy, especially for the fully connected layer, which has a similarity drop from around $1.0$ to $0.8$ under the number of copies $2,000$; 2) The shadow regions of A-Hybrid are much thinner than those of Hybrid, which indicates that the approximation  reduces the uncertainty of the vanilla method.    


\paragraph{Quantitative Analysis.}We propose quantitative metrics to give a specific comparison between different gradient estimation methods. Specifically,  for a given range of copies $[n_1,  n_2]$, we denote the gradient estimation accuracy $\textit{Acc}$ as 
\begin{equation}
    \textit{Acc.}= \frac{1}{n_2-n_1}\sum_{n=n_1}^{n_2}\cos{(g_n, g)}, 
    \label{eq:q_metric}
\end{equation}
where $g_n$ is the estimated gradient using the number of copies $n$,  and $g$ is the true gradient computed by the BP. A higher value of \textit{Acc.} indicates a more accurate gradient estimation.

Considering the consistency of the estimation method corresponding to the number of copies, we propose to use the standard error to evaluate the stability, 
\begin{equation}
    \textit{Sta.}= \sqrt{\frac{1}{n_2-n_1}\sum\limits_{n=n_1}^{n_2}(\cos{(g_n, g)-\hat{y}})^2}, 
    \label{eq:re}
\end{equation}
where we denote $\hat{y}$ as the ideal cosine similarity of gradient estimation outputted by the fitted linear regression model on observed  $\cos{(g_n, g)}$. A lower value of \textit{Sta.} indicates a better performance, where a large number of copies can consistently bring an improved gradient estimation performance.

For training the ResNet-5 on the CIFAR-10 dataset, we can compute the gradient estimation accuracy and stability for different methods in Tab. \ref{tab:tab_acc}. It can be observed that the approximation only introduces a little performance drop with an average of $0.03$ on the gradient estimation accuracy, while it improves the stability with an average of 0.04. Especially, it has an improvement of $0.05$ for LR, and $0.05$ for Hybrid, which contributes  to the neural network training process.
\input{tabs/acc}
\input{Figs/viz}
\paragraph{Visualization of optimization.} To further observe the behavior of different methods, we visualize their optimization trajectories of minimizing Beale’s function, i.e.,
\begin{equation*}
    \begin{aligned}
        \min\limits_{(x,y)\in\mathbb{R}} \mathcal{L}(x,y) = \big(\frac{3}{2} - x + xy\big)^2 + \big(\frac{9}{4} - x + xy^2\big)^2 + \big(\frac{21}{8} - x + xy^3\big)^2,
    \end{aligned}
\end{equation*}
where we initialize $(x,y)$ as $(-3,2)$, and represent the optimization variables as the weights of a single-layer neural network with a constant input of $1$ and a 2-dimensional output for all three methods.

We present the optimization trajectories of BP, LR, and ALR in Fig. \ref{fig:viz_sign_lr}. Each trajectory is color-coded: the \textcolor{blue}{blue line} represents BP, the \textcolor{orange}{orange line} corresponds to LR, and the \textcolor{green}{green line} signifies ALR. Beginning with the same initial conditions, a clear distinction emerges. ALR follows a more direct and expedited path, in contrast to BP and LR, which exhibit nearly identical optimization trajectories and are prone to convergence at sub-optimal points. This observation suggests that the integration of the approximated technique holds the potential to enhance exploratory capabilities, facilitating the model's traversal across multiple local optima.

\section{Related Work}
Numerous approaches have been proposed as potential alternatives to BP. These approaches can be broadly classified into three categories:  inter-layer dependence reduction, utilization of auxiliary variables, and perturbation-based gradient estimation.

For layer-wise dependence minimization, the HSIC bottleneck~\citep{ma2020hsic}  maximizes the corresponding measure to improve the independence of the feature representation between layers.  For the use of auxiliary variables,  the FA~\citep{nokland2016direct} uses learned feedback weight matrices to dispute the use of the chain rule in backward pass. the neural tangent kernel~\citep{jacot2018neural} employs a kernel function to model the relationship between the perturbation  and gradient. For perturbation-based gradient estimation, The ES~\citep{salimans2017evolution} random perturbs the weights and selects the better weights with minimal loss for  training. The LR method~\citep{peng2022new}  injects the noise on the intermediate features for the gradient estimation. A recent work~\citep{jiang2023training} extends LR to a wide range of architectures and points out that the ES can be derived within the LR technical framework by perturbing  weights.

In this paper, we align our work with the LR method. Recognizing the significant memory consumption challenges stemming from the proliferation of numerous copies, a factor that impedes the practical implementation of LR, we introduce an approximate LR approach. This method aims to mitigate memory-related concerns and enhance gradient estimation performance, particularly when dealing with a substantial number of copies.

\section{Conclusion}

In this study, to alleviate the problem of enormous memory consumption of gradient estimation techniques, we propose a novel approximated likelihood ratio method for enhancing neural network training. By discarding the storage of intermediate variables in high precision and retaining only their sign, we achieve a substantial reduction in both memory consumption and computational complexity.   Besides, by fully deploying the inherent parallelism of LR, we further propose a forward-only parallel pipeline to boost the efficiency of neural network training. Extensive experiments on various datasets and networks thoroughly validate the effectiveness of the proposed approximation technique and pipeline. We hope this study can motivate more future work on finding alternatives to BP for better efficiency, explainability, and biological plausibility.

\bibliographystyle{named}
\bibliography{ijcai24}

\end{document}


\maketitle
\appendix

\section{Experimental Details}\label{appendix: Experimental Details}
\noindent \textbf{Platform}: We conducted experiments on a computational platform with PyTorch 1.14.0 and eight NVIDIA RTX A6000 GPUs. Each A6000 GPU has 48 GB of memory.

\noindent \textbf{Datasets}: In our paper, we focus on the classification task. We conduct experiments {of ResNet-5} on the CIFAR-10 dataset, which consists of $60,000$ $32\times32$ color images in $10$ classes, with $5,000$ images for training and $1,000$ for testing per class. We train the ResNet-9 on the  CIFAR-100 dataset, which has $100$ classes consisting of $600$ images each. We train the ResNet-12 on the Tiny ImageNet dataset, which contains $100,000$ images of 200 classes.
{For experiments on the LSTM, we evaluate the proposed method on the Ag-News dataset with $4$ classes, which consists of  $30,000$  in news articles for training and $1,900$ for testing in each class. For the GAT, we consider the Cora dataset with $7$ classes, which consists of $140$ nodes for training and $1,000$ nodes for testing. For SNN, we conduct experiments on the MNIST and Fasion-MNIST datasets, which consist of $60,000$ $28\times28$  grey images for training and $10,000$ images for testing.

\noindent \textbf{Training}: For training of all models, we use the Adam optimizer with the learning rate $1\times10^{-3}$. We set the batch size as $64$, and the number of maximum epochs as $100$ for all optimization methods and models until fully converge of the loss function value on the validation dataset. It should be noted all compared methods, including BP, ES, LR, HSIC, FA, only differ in the computation of ``gradients'', which can be still integrated into conventional optimization methods, including SGD, Adam, and so on.

%% file: Figs/intro.tex
\begin{figure}
    \centering
    \includegraphics[width=0.8\linewidth]{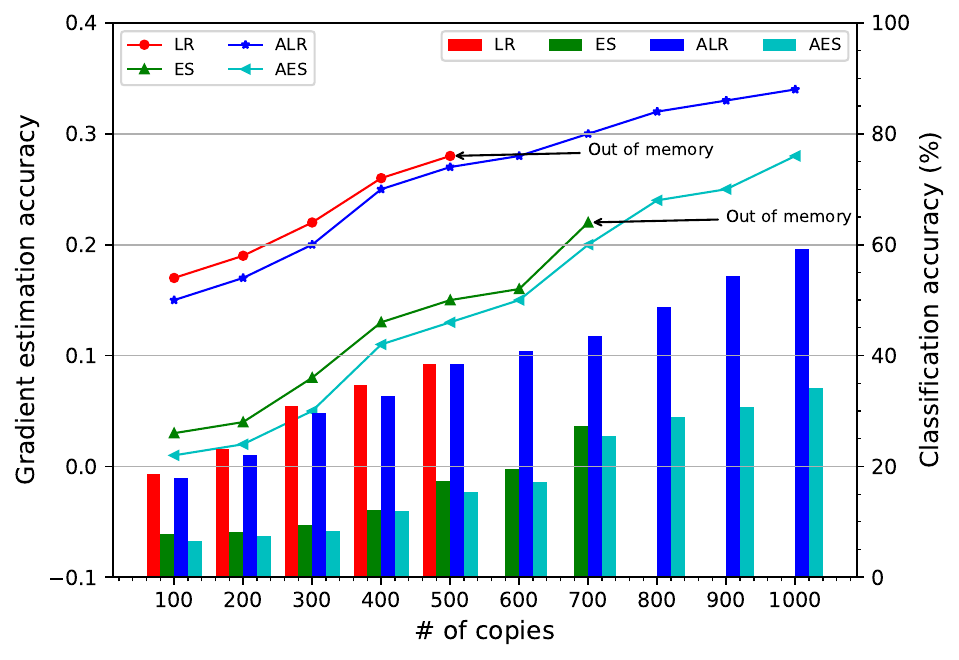}
    \vspace*{-0.1truein}
    \caption{Training ResNet-9 on the CIFAR-100 dataset using LR, ES, and corresponding approximated methods, ALR and AES.}
    \label{fig:intro}
\end{figure}

%% file: Figs/pipeline.tex
\begin{figure*}[t]
\centering
    \includegraphics[width=0.8\linewidth]{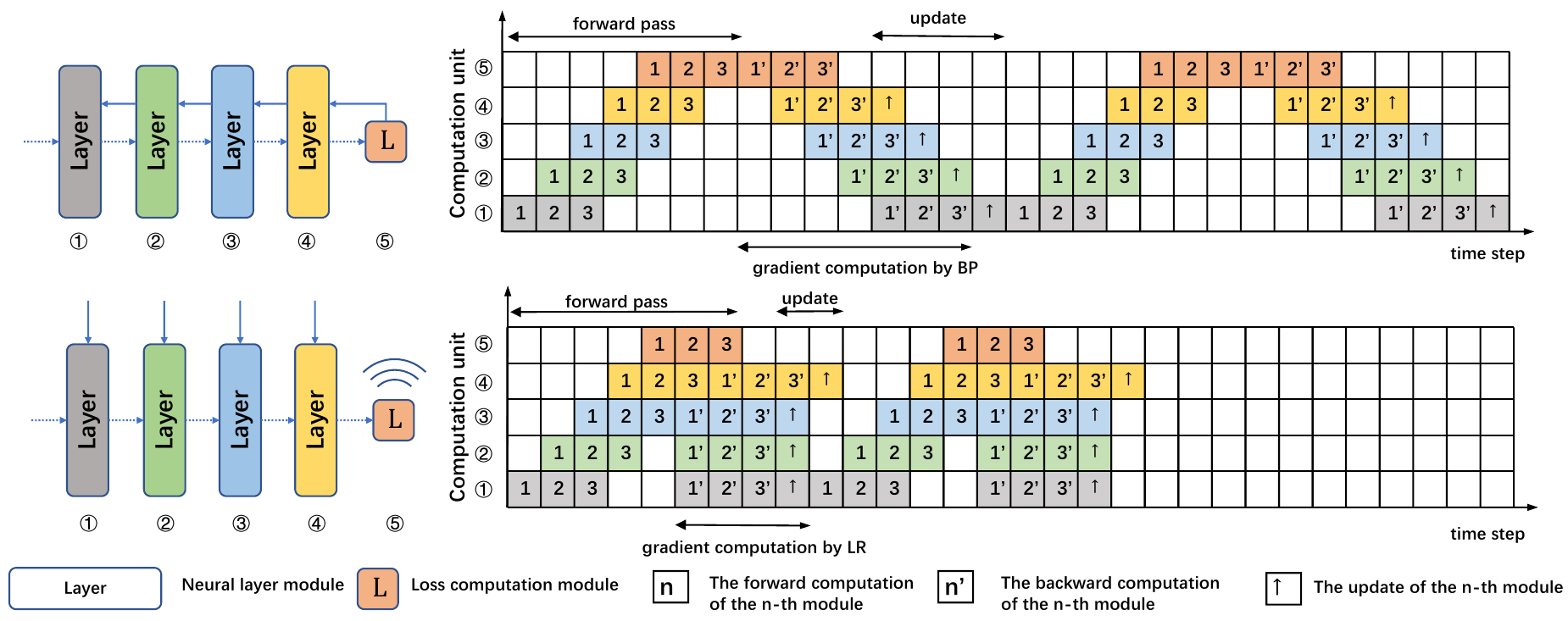}
    \caption{The design of hardware-efficient LR training, which pipelines both forward and gradient computation process. }
    \label{fig:pipeline}
    \vspace{-2mm}
\end{figure*}


%% file: Figs/cifar_verification_train.tex
\begin{figure}
\centering
\vspace{-2.5mm}
\hspace{-7mm}
\subfloat[Classification accuracy]{
    \begin{minipage}[b]{0.41\linewidth}
    \includegraphics[width=4.0cm]{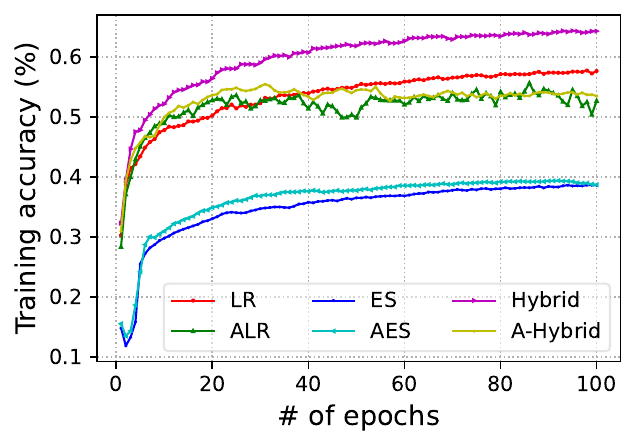}\vspace{-1pt} 
    \includegraphics[width=4.0cm]{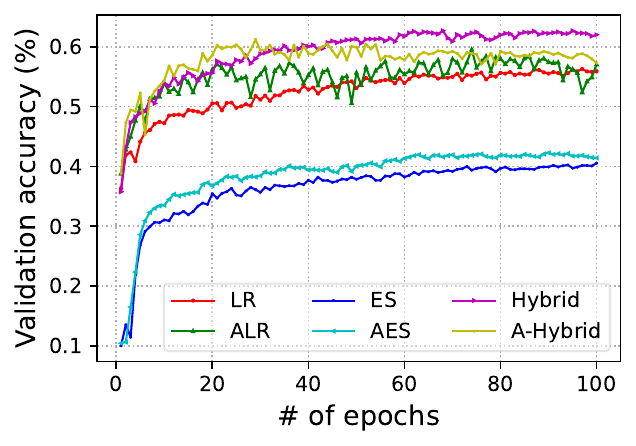}\vspace{-1pt}
    \end{minipage}
}
\quad 
\subfloat[Loss value]{
    \begin{minipage}[b]{0.41\linewidth}
    \includegraphics[width=4.0cm]{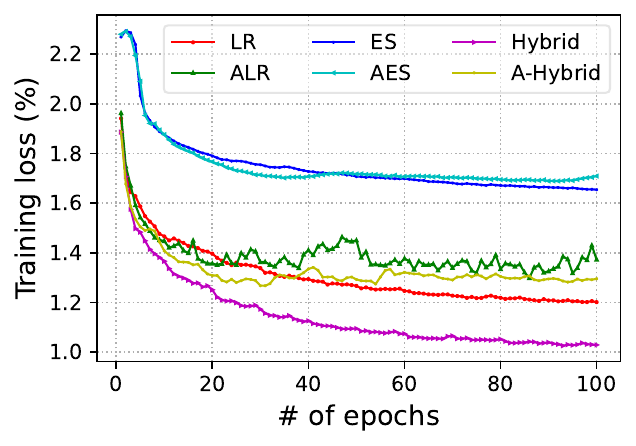}\vspace{-1pt} 
    \includegraphics[width=4.0cm]{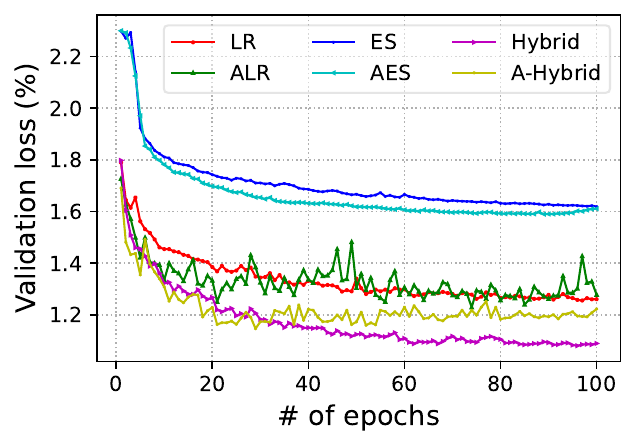}\vspace{-1pt}
    \end{minipage}
}
\vspace*{-0.1truein}
\caption{Learning curves of ResNet-5 on CIFAR-10. }
\label{fig:veri_cifar}
\end{figure}

%% file: tabs/comp_cifar100.tex
\begin{table}[]
    \centering
    \caption{Classification accuracies of ResNet-9 on the CIFAR-100 dataset using different optimization methods.}
    \label{tab:comp_cifar100}
     \resizebox{\linewidth}{!}{
    \begin{tabular}{cccccccc}
         \toprule
         Methods& BP& HSIC & FA  & ES & LR  & Hybrid   \\
         \hline
         w/o A&  65.2& 35.7 & 25.1  & 27.3 &38.5 & 42.6 \\
         \textbf{w A} &  \textbf{62.7}&  \textbf{34.9} & \textbf{23.8} &  {$\bf{34.1}$} & {$\bf{59.2}$} &  {$\bf{64.9}$}\\
         \bottomrule
    \end{tabular}}
\end{table}

%% file: tabs/comp_tinyImageNet.tex
\begin{table}[]
    \centering
    \caption{Classification accuracies of ResNet-12 on the Tiny-ImageNet dataset using different optimization methods.}
    \label{tab:comp_tiny}
     \resizebox{\linewidth}{!}{
    \begin{tabular}{cccccccc}
         \toprule
         Methods& BP& HSIC & FA  & ES & LR  & Hybrid   \\
         \hline
         w/o A&  49.8& 18.2 & 10.5  & 22.1 & 30.6 & 35.9 \\
         \textbf{w A} &  \textbf{48.2}&  \textbf{16.3} & \textbf{8.2} &  {$\bf{30.5}$} & {$\bf{43.5}$} &  {$\bf{48.3}$}\\
         \bottomrule
    \end{tabular}}
\end{table}

%% file: tabs/evaluation.tex
\begin{table}[]
    \centering
    \caption{Comparison between different optimization methods without or with integrating the approximation on various architectures.}
    \label{tab:multi_architecures}
     \resizebox{0.8\linewidth}{!}{
    \begin{tabular}{c|ccc}
    \toprule
         Model& BP   & ES & LR  \\
         \hline
          LSTM & 91.2/88.4  & 85.4/85.6& 89.7/\textbf{90.6}\\
         GAT   & 81.6/80.3  & 34.4/34.8 &81.8/\textbf{82.5}\\
          SNN  & 90.5/87.9 & 74.6/75.1 &92.3/\textbf{93.3} \\
    \bottomrule
    \end{tabular}}
\end{table}

%% file: Figs/accelerate.tex
\begin{figure}
    \centering
    \includegraphics[width=0.70\linewidth]{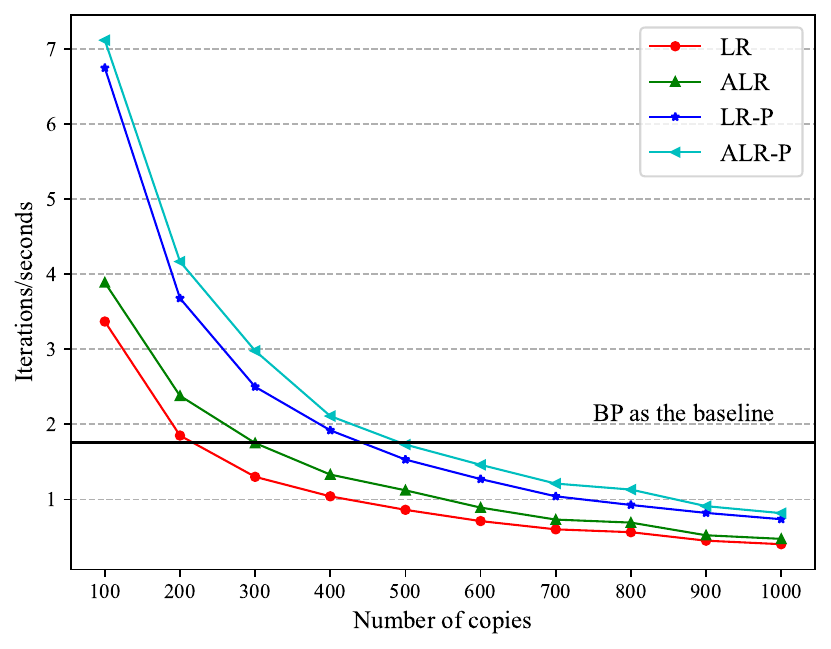}
    \vspace*{-0.15truein}
    \caption{Running efficiency of the LR training with integration of approximation (A-) and pipeline (-P). The black line indicates the number of iterations processed per second  using BP. }
    \label{fig:acc}
\end{figure}

%% file: Figs/trend.tex
\begin{figure}
  \centering
  \includegraphics[width=0.8\linewidth]{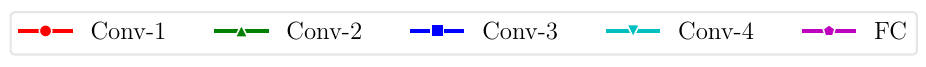}\\
  \includegraphics[width=0.48\linewidth]{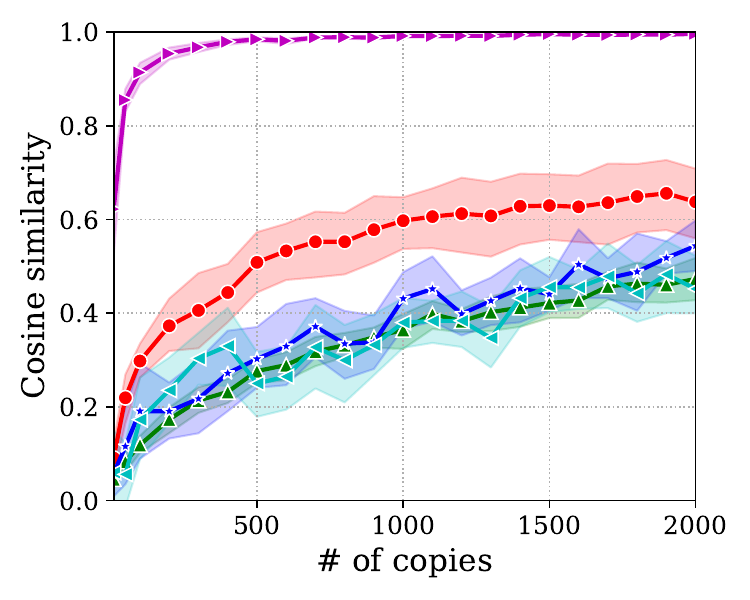}
  \includegraphics[width=0.48\linewidth]{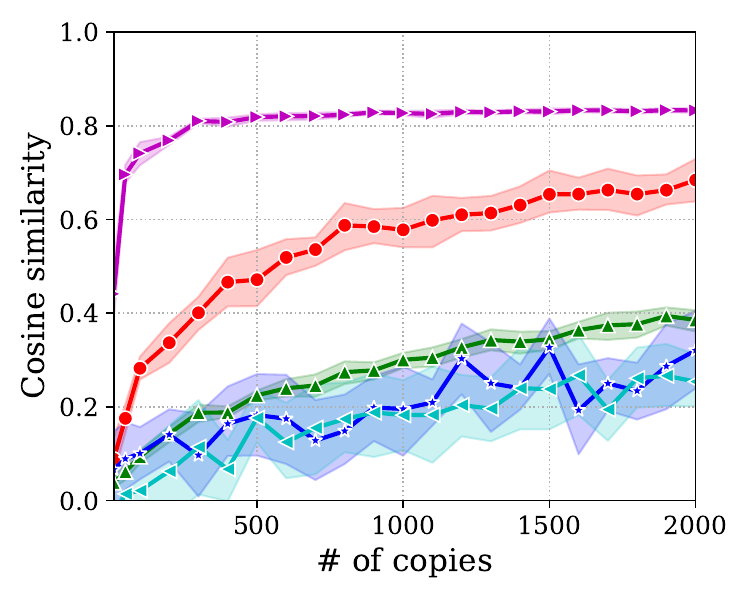}
  \vspace*{-0.15truein}
  \caption{From left to right, gradient cosine similarities between Hybrid and BP, A-Hybrid and BP. }
  \label{fig:gradient_acc}
\end{figure}

%% file: tabs/acc.tex
\begin{table}[]
    \centering
    \caption{Gradient estimation accuracy (Acc.) and stability (Sta.) of different gradient estimation methods. }
    \label{tab:tab_acc}
    \resizebox{.45\textwidth}{!}{
    \begin{tabular}{c|cccccc}
    \toprule
         Methods & LR& ALR & ES &AES & Hybrid &A-Hybrid  \\
         \hline
         \textit{Acc. $\uparrow$}&  0.36 & 0.34 & 0.22 & 0.18 & 0.42 & 0.39  \\
          \textit{Sta. $\downarrow$}&  0.21 & 0.16 & 0.06 & 0.05 & 0.18 & 0.13\\
         \bottomrule
    \end{tabular}}
\end{table}

%% file: Figs/viz.tex
\begin{figure}[t]
    \centering
    \includegraphics[width=0.9\linewidth]{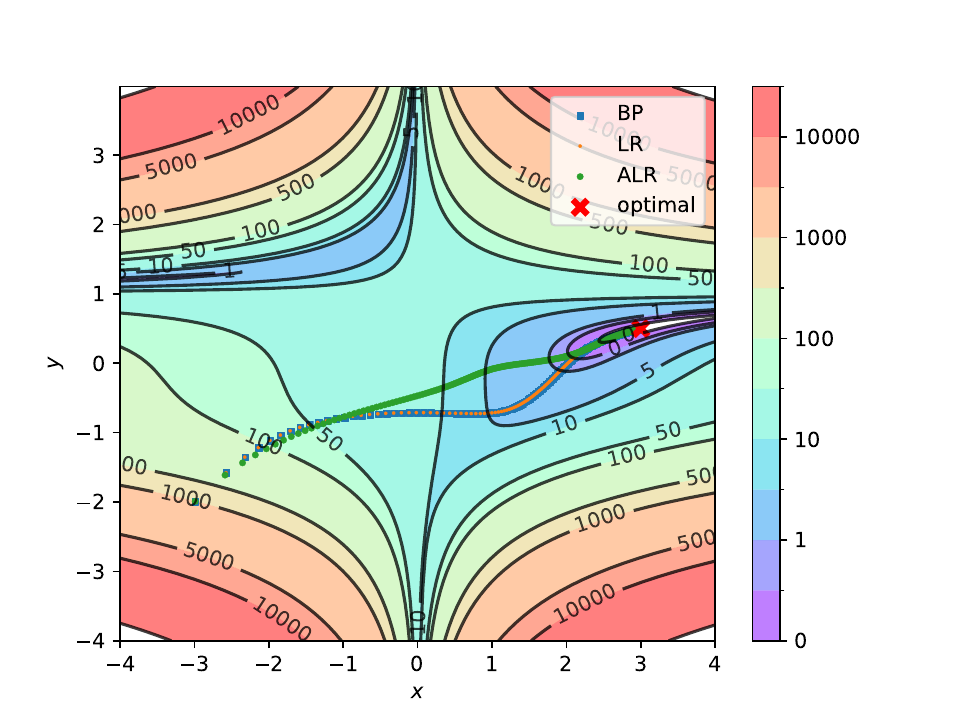}
    \vspace*{-0.1truein}
    \caption{Visualization of different optimization trajectories. For the 2-D optimization problem, we use \textcolor{blue}{blue line} to indicate the optimization trajectory of BP,  \textcolor{orange}{orange line} for LR, and \textcolor{green}{green line} for ALR.}
    \label{fig:viz_sign_lr}
\end{figure}